\documentclass[12pt]{l4dc2022}

\usepackage{bm}
\usepackage{amsmath}
\usepackage{mathtools}
\usepackage{wrapfig}
\usepackage{caption}
\graphicspath{ {figures/} }
\usepackage{longfbox}
\usepackage{verbatim}
\usepackage{booktabs,threeparttable}

\newtheorem{assumption}{Assumption}

\DeclarePairedDelimiter{\norm}{\lVert}{\rVert}
\DeclareMathOperator{\diag}{diag}
\DeclareMathOperator{\vect}{vec}
\DeclareMathOperator{\lcp}{LCP}
\newcommand{\defeq}{\vcentcolon=}
\newcommand{\qed}{$\blacksquare$}
\newcommand*{\tran}{^{\mkern-1.5mu\mathsf{T}}}
\newcommand{\mathleft}{\@fleqntrue\@mathmargin\parindent}
\newcommand{\mathcenter}{\@fleqnfalse\@mathmargin\parindent}

\title[Learning Linear Complementarity Systems]{Learning Linear Complementarity Systems}

\author{%
 \Name{Wanxin Jin} \Email{jinwx@seas.upenn.edu}\\
 \Name{Alp Aydinoglu} \Email{alpayd@seas.upenn.edu}\\
 \Name{Mathew Halm} \Email{mhalm@seas.upenn.edu}\\
 \Name{Michael Posa} \Email{posa@seas.upenn.edu} \\
 \addr University of Pennsylvania, Philadelphia, PA 19104, USA%
}

\begin{document}

\maketitle

\begin{abstract}%
This paper investigates the learning, or system identification, of a class of piecewise-affine dynamical systems known as linear complementarity systems (LCSs). We propose a violation-based loss which enables efficient learning of the LCS parameterization, without prior knowledge of the hybrid mode boundaries, using gradient-based methods. The proposed violation-based loss incorporates both dynamics prediction loss and a novel complementarity - violation loss. We show several properties attained by this loss formulation, including its differentiability, the efficient computation of first- and second-order derivatives, and its relationship to the traditional prediction loss, which strictly enforces complementarity. We apply this violation-based loss formulation to learn LCSs with tens of thousands of (potentially stiff) hybrid modes. The results demonstrate a state-of-the-art ability to identify piecewise-affine dynamics, outperforming methods which must differentiate through non-smooth linear complementarity problems.

\end{abstract}


\section{Introduction}
Many physical systems of interest are well captured by multi-modal or hybrid representations. For example, robotics problems which treat contact with the environment  \citep{stewart2000implicit, brogliato1999nonsmooth}, optimal control problems \citep{bemporad2000piecewise}, and control of networks \citep{heemels2002dynamic}, all exhibit switching or hybrid properties.

In this work, we are interested in system identification/model learning of multi-modal systems. We focus on piecewise-affine (PWA) models as they can sufficiently describe the multi-modal nature of dynamics due to the approximation properties of affine functions \citep{breiman1993hinging, lin1992canonical} but are tractable enough for control tasks due to their simple (affine) structure over polyhedral regions \citep{bemporad1999control}. Even though PWA models are widely used, it is well-known that PWA regression is NP-hard in general (see \citep{lauer2015complexity} for a detailed analysis), because it requires simultaneous classification of the data points into modes and the regression of a submodel for each mode.

In this paper,  we consider  PWA models in the context of  linear complementarity systems (LCSs) \citep{heemels2000linear}. We focus on a subclass of LCS models (with P-matrix assumption) that are equivalent to continuous piecewise affine models \citep{heemels2001equivalence, camlibel2007lyapunov}. LCS models are efficient representations of PWA models and an LCS has the ability to represent/approximate a system with large number of hybrid modes compactly, with only few complementarity variables. In some cases, an LCS with $n_\lambda$ complementarity variables is equivalent to a PWA model with $2^{n_\lambda}$ modes. Many robotics problems that involve contact can be efficiently locally approximated via LCS models, e.g.,  we have exploited  the LCS representation to enable contact-aware \citep{aydinoglu2021stabilization} and real-time control of robotic tasks \citep{aydinoglu2021real}.
In this work, we propose an approach that  learns an LCS from state-input data of a hybrid system, which does not contain any prespecified number of modes. The approach is able to identify LCS models by proposing an implicit loss function.

\subsection{Related Work}

Many successful approaches in identifying piecewise models have been proposed over the years. See the survey paper \citep{paoletti2007identification} for a detailed overview. Mixed integer formulations that mainly focus on hinging hyperplanes and piecewise affine Wiener models have been proposed \citep{roll2004identification} but as the number of integer variables scale with number of data points such approaches are only applicable in small data regime. On the contrary, researchers have also focused on convex formulations where first they estimate a set of submodels and then select few of them that explains the data \citep{elhamifar2014adaptive}, but the approach relies on restricting the parameter space and can be overly conservative.
Many alternate approaches that enable PWA system identification from data exist such as \citep{ferrari2003clustering, nakada2005identification,bemporad2005bounded,hartmann2015identification,du2020semi}. 
Researchers also suggested recursive PWA identification algorithms \citep{bako2011recursive}, \citep{breschi2016piecewise}. 
Most of the above methods are clustering-based where a predetermined number of models are identified and each training data point is associated with one of the models. Then, linear separation techniques are used to compute the polyhedral partitions. This iterative nature of the methods can lead to overly conservative, suboptimal solutions. Approaches that simultaneously cluster, PWL-separate and fit \citep{bemporad2021piecewise} rely heavily on initial assignment of data points to clusters. Unlike our approach, none of the methods above have been tested on identifying PWA functions with thousands of partitions, and most of them have been only tested on functions with less than $30$ pieces.

For more expressive models such as deep neural networks, researchers have explored the positive effect of imposing structured knowledge to capture the multi-modality \citep{de2018end, li2018learning, battaglia2016interaction}. Particularly in robotics, special emphasis has been on multi-body systems with frictional contact \citep{geilinger2020add} and it has been shown that imposing structure leads to accurate, sample efficient strategies \citep{pfrommer2020contactnets}. 
Similar work has demonstrated the difficulty inherent in learning non-smooth dynamical systems without exploiting particular structures \citep{parmar2021fundamental}. Researchers have also explored learning models as functionals of signed-distance fields \citep{driess2021learning}. These methods lead to rich, accurate but complex models that are not amenable to techniques of model-based control. On the contrary, here our focus is on simple models such as PWA models that enable model-based control while sufficiently capturing the hybrid dynamics.

\bigskip

\noindent
\textbf{Notation:} \quad In this paper, regular and bold lowercase letters  represent scalar and vectors, respectively. Uppercase letters represent matrices. For vector $\boldsymbol{v}\in\mathbb{R}^n$, $\boldsymbol{v}[i]$ is the $i$-th entry, $i=1,2,...,n$.  $\diag(\boldsymbol{v})$ is to diagonalize a vector $\boldsymbol{v}$ into a matrix.   $\vect{(A)}$  denotes the vectorlization of a matrix ${A}$ into a column vector; $\otimes $ is the Kronecker product. $I_n$ denotes the identity matrix with size of $n\times n$.   $A\succ 0$ means   symmetric  $A$ is positive definite.

\section{Problem Statement}
Consider the following discrete-time linear complementarity system (LCS), where the state evolution is governed by  a linear  dynamics in (\ref{equ.lcs.1}) and a linear complementarity problem (LCP) in  (\ref{equ.lcs.2}):
\begin{subequations}\label{equ.lcs}
	\begin{align}
	&\boldsymbol{x}_{t+1}=A\boldsymbol{x}_t+B\boldsymbol{u}_t+C\boldsymbol{\lambda}_t+\boldsymbol{d}, \label{equ.lcs.1}\\
	&\boldsymbol{0}\leq\boldsymbol{\lambda}_t \perp D\boldsymbol{x}_t  +E\boldsymbol{u}_t + F\boldsymbol{\lambda}_t + \boldsymbol{c} \geq \boldsymbol{0}. \label{equ.lcs.2}
	\end{align}
\end{subequations}
Here, $\boldsymbol{x}_t\in \mathbb{R}^{n_x}$ and $\boldsymbol{u}_t\in \mathbb{R}^{n_u}$ are the system state and input at time step $t$, respectively; and  $\boldsymbol{\lambda}_t\in\mathbb{R}^r$ is the complementarity variable at time step $t$.   $A\in\mathbb{R}^{{n_x}\times {n_x} }$,  $B\in\mathbb{R}^{{n_x}\times {n_u} }$, $C\in\mathbb{R}^{{n_x}\times {n_\lambda} }$,
$\boldsymbol{d}\in\mathbb{R}^{n_x} $, $D\in\mathbb{R}^{{n_\lambda} \times {n_x}  }$, $E\in\mathbb{R}^{{n_\lambda} \times n_u }$, $F\in\mathbb{R}^{{n_\lambda} \times {n_\lambda}  }$,  and $\boldsymbol{c}\in\mathbb{R}^{n_\lambda} $ are system  matrix/vector parameters. At $(\vec{x}_t, \vec{u}_t)$,    $\boldsymbol{\lambda}_t$ is solved from the LCP in  (\ref{equ.lcs.2}), written as
\begin{equation}\label{equ.lcp}
\boldsymbol{\lambda}_t\in\lcp(F, \vec{q}_t) \quad \text{with} \quad \vec{q}_t:=D\boldsymbol{x}_t  +E\boldsymbol{u}_t+\boldsymbol{c}.
\end{equation}
It is well-known that    $\lcp(F, \vec{q}_t)$ has a unique solution $\boldsymbol{\lambda}_t$ for every $\vec{q}_t$ if and only if $F$ is $P$-matrix \citep{cottle2009linear}. We will discuss this  in  the next section.

In this paper, we consider to learn a LCS  from a dataset  $\mathcal{D}=\{(\boldsymbol{x}_t^*,\boldsymbol{u}_t^*, \boldsymbol{x}_{t+1}^*)\}_{t=1}^N$. Specifically, we aim to find the system parameter
\begin{equation}
\boldsymbol{\theta}=\{A,B,C,\boldsymbol{d},D,E,F,\boldsymbol{c}\}
\end{equation}
by minimizing a  loss  
 $L(\boldsymbol{\theta},\mathcal{D})$. Thus, the problem of interest in this paper is to solve
\begin{equation}\label{equ.loss}
\min_{\boldsymbol{\theta}\in\boldsymbol{\Theta}} \quad L(\boldsymbol{\theta},\mathcal{D}) + {R} (\boldsymbol{\theta}).
\end{equation}
Here, $R(\boldsymbol{\theta})$ can be any regularization term imposed on $\boldsymbol{\theta}$ which will be discussed  later.

\section{LCP and Prediction-based Formulation} \label{section.basics}

This section will discuss the solution to  LCP in (\ref{equ.lcp}), and then  describe a prediction-based loss formulation $L(\boldsymbol{\theta},\mathcal{D})$ for (\ref{equ.loss}). To start, we make the following assumption on
  $F$ in (\ref{equ.lcp}).
\begin{assumption}\label{assumption1}
	 $F\in\mathbb{R}^{n_\lambda\times n_\lambda}$ satisfies $F+F\tran\succ0$.
\end{assumption}
The set of  $F$ satisfying Assumption \ref{assumption1} contains all positive matrices of  feasible dimension and  any asymmetric matrices with definite-positive symmetric part. In robotics applications, $F$ with Assumption \ref{assumption1} has been widely used in soft contact dynamics problems such as [TBD].  
Any  $F$ satisfying Assumption \ref{assumption1} can be shown to be a  $P$-matrix \citep{tsatsomeros2002generating}, thus leading to the existence and  uniqueness of $\vec{\lambda}_t$. In fact, under Assumption \ref{assumption1},  $\boldsymbol{\lambda}_t=\lcp(F, \vec{q}_t)$ can be solved by the following convex optimization due to the fact $\boldsymbol{\lambda}\tran F\boldsymbol{\lambda}=\frac{1}{2}\boldsymbol{\lambda}\tran(F+F\tran)\boldsymbol{\lambda}$,
\begin{equation}\label{equ.convexlcp}
\begin{aligned}
\boldsymbol{\lambda}_t=\arg\,\,&\min_{\boldsymbol{\lambda}} \,\,
\frac{1}{2}\boldsymbol{\lambda}\tran(F+F\tran)\boldsymbol{\lambda}+\boldsymbol{\lambda}\tran\vec{q}_t\qquad\text{s.t.} \quad F\boldsymbol{\lambda}+\vec{q}_t\geq \boldsymbol{0}, \quad \quad \boldsymbol{\lambda}\geq \boldsymbol{0},
\end{aligned}
\end{equation}
With the above assumption,  one  natural   loss  in (\ref{equ.loss})  can be
\begin{equation}\label{equ.prednet}
L^{\text{pred}}(\boldsymbol{\theta},\mathcal{D}){=}\,\,\,
\sum_{t=1}^{N}\,\,\,\frac{1}{2}
\norm{\boldsymbol{x}^{\boldsymbol{\theta}}_{t+1}-\boldsymbol{x}_{t+1}^*}^2 \quad \text{with} \quad
\begin{aligned}
&\boldsymbol{x}_{t+1}^{\boldsymbol{\theta}}=A\boldsymbol{x}_t^*+B\boldsymbol{u}_t^*+C\boldsymbol{\lambda}_t^*+\boldsymbol{d}, \\
&\boldsymbol{\lambda}_t^*=\lcp(F,D\boldsymbol{x}_t^*  +E\boldsymbol{u}_t^*+\vec{c}).
\end{aligned}
\end{equation}
Here, $\boldsymbol{x}_{t+1}^{\boldsymbol{\theta}}$ is the predicted next state, implicitly depending on  $\boldsymbol{\theta}$. We call (\ref{equ.prednet})  \emph{prediction-based loss}, as it evaluates the difference between the predicted $\boldsymbol{x}_{t+1}^{\boldsymbol{\theta}}$ and  observed $\boldsymbol{x}_{t+1}^*$. One can minimize  (\ref{equ.prednet}) via any gradient-based method   by differentiating through LCP \citep{de2018end}. This requires  differentiblity of a LCP,  given below.
\begin{lemma}\label{lemma.lcpdiff}
	With Assumption \ref{assumption1},  $\boldsymbol{\lambda}_t^*=\lcp(F,D\boldsymbol{x}_t^*  +E\boldsymbol{u}_t^*+\vec{c})$ is differentiable with respect to  $(F,D, E,\vec{c})$, if the following strict complementarity   holds at $\boldsymbol{\lambda}_t^*$:  $\boldsymbol{\lambda}_t^*[i]>0$ or  $(F\boldsymbol{\lambda}_t^*+D\boldsymbol{x}_t^*  +E\boldsymbol{u}_t^*+\vec{c})[i]>0$,  $\forall i=1,2,...,n_{\lambda}.$
\end{lemma}

A sketch of a  proof of Lemma \ref{lemma.lcpdiff} is given in Appendix \ref{appendix.proof.lemma1}. The above prediction-based loss in (\ref{equ.prednet})   will serve as a benchmark in the following method development.



\section{Proposed Method for Learning LCS}
This section will develop a new  method for learning LCS. As we will show this section and the experiments in next section, the proposed method  attains several  advantages over the prediction-based loss (\ref{equ.prednet}) both in theoretical property and implementation. 

\subsection{Violation-based Loss}
To start, we give the following lemma  stating an equivalence of a LCP. 

\begin{lemma}\label{lemma.cn.key}
	Given any $\boldsymbol{q}_t\in\mathbb{R}^{n_\lambda}$ and   $F$ satisfying Assumption \ref{assumption1}, solving $\vec{\lambda}_t=\lcp(F,\vec{q}_t)$  is the equivalent to solving the following strongly-convex quadratic program:
	\begin{equation}\label{equ.cn.keyequ}
	(\boldsymbol{\lambda}_t, \boldsymbol{\phi}_t)=\arg\min_{\boldsymbol{\lambda}\geq\boldsymbol{0}, \,\,\boldsymbol{\phi}\geq\boldsymbol{0}} \quad \boldsymbol{\lambda}\tran\boldsymbol{\phi} +
	\frac{1}{2\gamma}
	\norm{F\boldsymbol{\lambda}+\boldsymbol{q}_t-\boldsymbol{\phi}}^2,
	\end{equation}
	with any  constant $0<\gamma<\sigma_{\min}(F\tran{+}F)$ (  $\sigma_{\min}(\cdot)$ denotes  the smallest singular value). 
\end{lemma}

\begin{proof}
Define $f(\boldsymbol{\lambda},\boldsymbol{\phi})\defeq\boldsymbol{\lambda}\tran\boldsymbol{\phi} +
\frac{1}{\gamma}
\norm{F\boldsymbol{\lambda}+\boldsymbol{q}_t-\boldsymbol{\phi}}^2.$ By non-negativity, it is obvious that $\vec{\lambda}_t=\lcp(F,\vec{q}_t)$ and $\boldsymbol{\phi}_t=F\vec{\lambda}_t+\vec{q}_t$ is a global solution to $f(\boldsymbol{\lambda},\boldsymbol{\phi})$. Further, we need to show that $(\boldsymbol{\lambda}_t, \boldsymbol{\phi}_t)$ is a unique solution to  (\ref{equ.cn.keyequ}). To do so, we  compute the Hessian of  $f(\boldsymbol{\lambda},\boldsymbol{\phi})$,
\begin{equation}
\nabla^2 f=
\begin{bmatrix}
\frac{1}{\gamma} F\tran F & I_{n_\lambda}-\frac{1}{\gamma} F\tran\\[4pt]
I_{n_\lambda}-\frac{1}{\gamma}F &  \frac{1}{\gamma}  I_{n_\lambda}
\end{bmatrix}.
\end{equation}
Due to  Schur complement, $\nabla^2 f\succ0$  iff $\frac{1}{\gamma}I_{n_\lambda}\succ  0$ and
 $\frac{1}{\gamma} F\tran F-(I_{n_\lambda}-\frac{1}{\gamma} F\tran)
(\frac{1}{\gamma}I_{n_\lambda})^{-1}(I_{n_\lambda}-\frac{1}{\gamma}F)\succ 0$. Since $\gamma>0$,  and we only need to show
$
\small
\frac{1}{\gamma} F\tran F-(I_{n_\lambda}-\frac{1}{\gamma} F\tran)
(\frac{1}{\gamma}I_{n_\lambda})^{-1}(I_{n_\lambda}-\frac{1}{\gamma}F)=F\tran+F-\gamma I_{n_\lambda}
\succ 0.
$
This is true because $\gamma{<}\sigma_{\min}(F\tran{+}F)$. 
This completes the proof. 
\end{proof}

In  (\ref{equ.cn.keyequ}), we have introduced a proxy variable $\boldsymbol{\phi}\geq \boldsymbol{0}$ to represent LCP constraint $F\boldsymbol{\lambda}+\vec{q}_t\geq \boldsymbol{0}$. Compared to other  equivalences of LCP, such as (\ref{equ.convexlcp}),  we emphasize the following benefits of (\ref{equ.cn.keyequ}) with the introduced proxy variable $\boldsymbol{\phi}$. First,  (\ref{equ.cn.keyequ}) now  only has  box constraints which are \emph{independent} from  $\vec{\theta}$; this will facilitate the  learning  process  because one does not need to explicitly track the  active and inactive constraints and differentiate through constraints (which usually leads to numerical difficulty as shown in \citep{jin2021safe}). Second, compared to (\ref{equ.convexlcp}),  (\ref{equ.cn.keyequ}) turns  hard constraint   $F\vec{\lambda}+\vec{q}_t\geq 0$ into a soft penalty; this may smooth the landscape of the proposed loss, facilitating the optimization process over $\boldsymbol{\theta}$. With  Lemma \ref{lemma.cn.key}, we  are now in a position to propose the following loss function for learning LCS,

\begin{longfbox}[padding-top=3pt,margin-top=5pt, padding-bottom=0pt, margin-bottom=5pt]
\begin{subequations}\label{equ.contactnet}
	\begin{equation}\label{equ.contactnet.1}
	L_{\epsilon}(\boldsymbol{\theta},  \mathcal{D})=\,\,\,
	\sum\nolimits_{t=1}^{N}
	l_{\epsilon}( \boldsymbol{\theta},\boldsymbol{x}_{t}^*,  \boldsymbol{u}_{t}^*, \boldsymbol{x}_{t+1}^*) \quad \text{with}
	\end{equation}
	\begin{equation}\label{equ.contactnet.2}
	\begin{aligned}
	l_{\epsilon}(\boldsymbol{\theta}, \boldsymbol{x}_{t}^*, \boldsymbol{u}_{t}^*, \boldsymbol{x}_{t+1}^*)=&\min_{\boldsymbol{\lambda}_t\geq\boldsymbol{0},\,\, \boldsymbol{\phi}_t\geq\boldsymbol{0}} \,\, \frac{1}{2}\norm{A\boldsymbol{x}_t^*+B\boldsymbol{u}_t^*+
		C\boldsymbol{\lambda}_t+\boldsymbol{d}-\boldsymbol{x}_{t+1}^{*}}^2+\\
	&\qquad \quad \qquad \frac{1}{\epsilon} \,\, 
	\left(
	\boldsymbol{\lambda}_t\tran\boldsymbol{\phi}_t
	+\frac{1}{2\gamma}
	\norm{D\boldsymbol{x}_t^*  +E\boldsymbol{u}_t^* + F\boldsymbol{\lambda}_t + \boldsymbol{c}-\boldsymbol{\phi}}^2
	\right),
	\end{aligned}
	\end{equation}
\end{subequations}
\end{longfbox}
with $\epsilon>0$. In $L_{\epsilon}(\boldsymbol{\theta},  \mathcal{D})$, the loss $l_{\epsilon}(\boldsymbol{\theta}, \boldsymbol{x}_{t}^*, \boldsymbol{u}_{t}^*, \boldsymbol{x}_{t+1}^*)$ on each data point $(\boldsymbol{x}_{t}^*, \boldsymbol{u}_{t}^*, \boldsymbol{x}_{t+1}^*)$ includes two parts: the violation of dynamics  (\ref{equ.lcs.1}) and the violation of the LCP, as stated in Lemma \ref{lemma.cn.key}.  We have introduced  parameter $\epsilon>0$ to control the weight of penalties on the two violations.  (\ref{equ.contactnet}) is to minimize data's violation to both  dynamics  (\ref{equ.lcs.1}) and  complementarity constraints (\ref{equ.lcs.2}), thus we name it \emph{violation-based loss}. In what follows, we will show that   the violation-based loss attains some good properties both for theoretical analysis and algorithmic implementation,  in comparison with  the prediction-based loss (\ref{equ.prednet}).

\subsection{Properties of Violation-based Loss}\label{section.property}

The first lemma shows that the violation-based loss (\ref{equ.contactnet}) is a strongly-convex quadratic program w.r.t. $(\boldsymbol{\lambda}, \boldsymbol{\phi})$ and allows  much easier computation of the gradient w.r.t. $\boldsymbol{\theta}$.

\begin{lemma}\label{lemma.contactnet}
	Given  $F$ satisfying Assumption \ref{assumption1} and any constant $0<\gamma< \sigma_{\min}(F\tran+F)$,
	\begin{itemize}
		\item[(a)] (\ref{equ.contactnet.2})  is strongly-convex quadratic program with respect to $(\vec{\lambda}_t,\vec{\phi}_t)$ for any $\epsilon>0$.
		\item[(b)] Let   $(\boldsymbol{\lambda}_t^{\epsilon,\boldsymbol{\theta}}, \boldsymbol{\phi}_t^{\epsilon,\boldsymbol{\theta}})$ be the solution to (\ref{equ.contactnet.2}). $L_{\epsilon}(\boldsymbol{\theta},\mathcal{D})$ is differentiable with respect to~$\boldsymbol{\theta}$  if the   strict complementarity holds for both $\boldsymbol{\lambda}_t{\geq}\boldsymbol{0}$ and $\boldsymbol{\phi}_t{\geq}\boldsymbol{0}$ at $(\boldsymbol{\lambda}_t^{\epsilon,\boldsymbol{\theta}}, \boldsymbol{\phi}_t^{\epsilon,\boldsymbol{\theta}})$, $t=1,2,...,N$. The gradient  is given by
		\begin{equation}\label{equ.diffvioloss}
		\small
		\begin{aligned}
		&\nabla_{A}L_{\epsilon}=\sum_{t=1}^{N}\boldsymbol{e}_t^{\emph{dyn}}{\boldsymbol{x}_t^*}\tran,
		\,\,
		\nabla_{B}L_{\epsilon}=\sum_{t=1}^{N}\boldsymbol{e}_t^{\emph{dyn}}{\boldsymbol{u}_t^*}\tran,
		\,\, 
		\nabla_{C}L_{\epsilon}=\sum_{t=1}^{N}\boldsymbol{e}_t^{\emph{dyn}}{\boldsymbol{\lambda}_t^{*,\boldsymbol{\theta}}}\tran,
		\,\, 
		\nabla_{\boldsymbol{d}}L_{\epsilon}=\sum_{t=1}^{N}\boldsymbol{e}_t^{\emph{dyn}},
		\\
		&
		\nabla_{D}L_{\epsilon}=\sum_{t=1}^{N}\boldsymbol{e}_t^{\emph{lcp}}{\boldsymbol{x}_t^*}\tran,
		\,\
		\nabla_{E}L_{\epsilon}=\sum_{t=1}^{N}\boldsymbol{e}_t^{\emph{lcp}}{\boldsymbol{u}_t^*}\tran, \,\,
		\nabla_{F}L_{\epsilon}=\sum_{t=1}^{N}\boldsymbol{e}_t^{\emph{lcp}}{\boldsymbol{\lambda}_t^{\epsilon,\boldsymbol{\theta}}}\tran,
		\quad
		\nabla_{\boldsymbol{c}}L_{\epsilon}=\sum_{t=1}^{N}\boldsymbol{e}_t^{\emph{lcp}},
		\end{aligned}
		\end{equation}
		with $\boldsymbol{e}_t^{\emph{dyn}}{:=}A\boldsymbol{x}_t^*{+}B\boldsymbol{u}_t^*{+}
		C\boldsymbol{\lambda}_t^{\epsilon,\boldsymbol{\theta}}{+}\boldsymbol{d}{-}\boldsymbol{x}_{t+1}^{*}$ and 
		$\boldsymbol{e}_t^{\emph{lcp}}{:=}\frac{1}{\epsilon\gamma}(D\boldsymbol{x}_t^*  {+}E\boldsymbol{u}_t^* {+} F\boldsymbol{\lambda}_t^{\epsilon,\boldsymbol{\theta}}{ +} \boldsymbol{c}{-}\boldsymbol{\phi}_t^{\epsilon,\boldsymbol{\theta}})$.

	\end{itemize}

\end{lemma}

\begin{proof}
	Claim (a) in Lemma \ref{lemma.contactnet} can be easily  proved by verifying that the Hessian  of the objective function in (\ref{equ.contactnet.2}) is positive definite.

	 In Claim (b), the differentiability of  $L_{\epsilon}(\boldsymbol{\theta},\mathcal{D})$ depends on the differentiability of $(\boldsymbol{\lambda}_t^{\epsilon,\boldsymbol{\theta}}, \boldsymbol{\phi}_t^{\epsilon,\boldsymbol{\theta}})$  with respect to~$\boldsymbol{\theta}$, $t=1,2,...,N$. In fact, $(\boldsymbol{\lambda}_t^{\epsilon,\boldsymbol{\theta}}, \boldsymbol{\phi}_t^{\epsilon,\boldsymbol{\theta}})$ is differentiable with respect to $\boldsymbol{\theta}$ if the strict complementarity condition holds for constrained optimization in (\ref{equ.contactnet.2}). This is a direct result from the well-known sensitivity analysis theory  (see Theorem 2.1 in \citep{fiacco1976sensitivity}). The gradient of $L_{\epsilon}(\boldsymbol{\theta},\mathcal{D})$ can be obtained directly applying the  envelope theorem \citep{afriat1971theory}. For example,
	the gradient of $l_{\epsilon}(\boldsymbol{\theta}, \boldsymbol{x}_{t}^*, \boldsymbol{u}_{t}^*, \boldsymbol{x}_{t+1}^*)$ with respect to matrix $A$ is 
	 \begin{equation*}
	 \begin{aligned}
	 \nabla_{ \vect{(A)}}l_\epsilon=
	 \left(\frac{d l_\epsilon}{d \vect{(A)}}\right)\tran=
	 \left(
	 (\boldsymbol{e}_t^{\text{dyn}})\tran\left(
	 {\boldsymbol{x}_t^*}\tran\otimes I_{n_x}
	 \right)
	 \right)\tran=
	 \left(
	 \boldsymbol{x}_t^*\otimes I_{n_x}
	 \right)\boldsymbol{e}_t^{\text{dyn}}= \vect(\boldsymbol{e}_t^{\text{dyn}}{\boldsymbol{x}_t^*}\tran).
	 \end{aligned}
	 \end{equation*}
	 Writing the above into the matrix form leads to $	 \nabla_{A}L_{\epsilon}=\sum_{t=1}^{N}\boldsymbol{e}_t^{\text{dyn}}(\boldsymbol{x}_t^*)\tran.$
	 Similar derivations also apply to $\nabla_{B}L_{\epsilon}$,
	 $\nabla_{B}L_{\epsilon}$,
	 $\nabla_{C}L_{\epsilon}$,
	 $\nabla_{D}L_{\epsilon}$,
	 $\nabla_{E}L_{\epsilon}$, and 
	 $\nabla_{F}L_{\epsilon}$.
	 This completes the proof. 
\end{proof}

In addition to the strongly-convex quadratic problem with bound constraints in (\ref{equ.contactnet.2}), Lemma \ref{lemma.contactnet} state that $L_{\epsilon}(\boldsymbol{\theta},\mathcal{D})$ allows for much simpler differentiation, as stated in claim (b). Note that  differentiation of $L_{\epsilon}(\boldsymbol{\theta},\mathcal{D})$  in (\ref{equ.diffvioloss}) does not involve any matrix inverse. This is in stark contrast with the prediction-based loss (\ref{equ.prednet}), whose differentiation  \citep{de2018end} is based on the implicit function theorem \citep{rudin1976principles}  and requires the inverse of Jacobian matrix of KKT equations (which is  computationally expensive).

Another implication of Lemma  \ref{lemma.contactnet} is that the Lipschitz constant of $L_{\epsilon}(\boldsymbol{\theta},\mathcal{D})$   with respect to the LCP matrices $(D, E, F, \vec{c})$ can be controlled by the choice of $\epsilon$. Specifically, the second line of (\ref{equ.diffvioloss}) shows that one can always choose a large $\epsilon$ to produce a small  Lipschitz constant of the loss landscape with respect to  $(D, E, F, \vec{c})$. This property can  facilitate the learning process by controlling the smoothness of the loss landscape,   and also helpful in the generalization of  learned results as analyzed in a concurrent work \citep{bianchini2021generalization}. However, we also need to note that the large choice of $\epsilon$ could lead to the bias learning results, as shown and analyzed in the later simulation examples.

We further have the following result, which states the  second-order derivative of  $L_{\epsilon}(\boldsymbol{\theta},  \mathcal{D})$. 

\begin{lemma}\label{lemma.cn.hessian}
Given  $F$ satisfying Assumption \ref{assumption1} and any constant $0<\gamma< \sigma_{\min}(F\tran+F)$, suppose that the differentiability  in Lemma \ref{lemma.contactnet} holds. Then,
\begin{equation}\label{equ.2nddiff}
\small
\nabla_{\boldsymbol{\theta}}^2\, {L}_{\epsilon}=\sum\limits_{t=1}^{N}\left(
\frac{\partial^2 L_{\epsilon}}{\partial \boldsymbol{\theta} \partial \boldsymbol{\theta}}-
\frac{\partial^2 L_{\epsilon}}{\partial \boldsymbol{\theta} \partial \boldsymbol{z}_t^{\epsilon}}
\left(
\diag\Big(\frac{\partial L_{\epsilon}}{\partial \boldsymbol{z}_t^{\epsilon}}\Big)+\diag(\boldsymbol{z}_t^{\epsilon})		\frac{\partial^2 L_{\epsilon}}{\partial \boldsymbol{z}_t^{\epsilon}\partial \boldsymbol{z}_t^{\epsilon} }
\right)^{{-}1}\diag(\boldsymbol{z}_t)	
\frac{\partial^2 L_{\epsilon}}{\partial \boldsymbol{z}_t^{\epsilon}\partial \boldsymbol{\theta} }\right),
\end{equation}
	with $\boldsymbol{z}_t^{\epsilon}=(\boldsymbol{\lambda}_t^{\epsilon,\boldsymbol{\theta}}, \boldsymbol{\phi}_t^{\epsilon,\boldsymbol{\theta}})$ being the solution to  (\ref{equ.contactnet.2}).
\end{lemma}

\begin{proof}
	To prove Lemma \ref{lemma.cn.hessian},  we need first to show   $	\left(
	\diag\Big(\frac{\partial L_{\epsilon}^{\text{}}}{\partial \boldsymbol{z}_t^{\epsilon}}\Big)+\diag(\boldsymbol{z}_t^{\epsilon})		\frac{\partial^2 L_{\epsilon}^{\text{}}}{\partial \boldsymbol{z}_t^{\epsilon}\partial \boldsymbol{z}_t^{\epsilon} }
	\right)$ is invertible. We only provide the sketch of the proof due to page limits. First, the KKT conditions at  $\boldsymbol{z}_t^{\epsilon}=(\boldsymbol{\lambda}_t^{\epsilon,\boldsymbol{\theta}}, \boldsymbol{\phi}_t^{\epsilon,\boldsymbol{\theta}})$ can be written as the following LCP
	\begin{equation}\label{KKT}
		\boldsymbol{0}\leq \small\left(\frac{\partial L_{\epsilon}^{\text{}}}{\partial \boldsymbol{z}_t^{\epsilon}}\right)^{\prime}\perp\boldsymbol{z}_t^{\epsilon}\geq \boldsymbol{0}.
	\end{equation}
	The    strict complementarity  (differentiblity of $L_{\epsilon}^{\text{}}$) stated in claim (b) of Lemma \ref{lemma.contactnet} is equivalent to say the above LCP in (\ref{KKT}) is strictly complementarity. By claim (a) in  Lemma \ref{lemma.contactnet}, we have known  $	\frac{\partial^2 L_{\epsilon}^{\text{}}}{\partial \boldsymbol{z}_t^{\epsilon}\partial \boldsymbol{z}_t^{\epsilon} }\succ 0$, which is a P-matrix. Following the same proof in Appendix  \ref{appendix.proof.lemma1}, one can  show that  $	\left(
	\diag\Big(\frac{\partial L_{\epsilon}^{\text{}}}{\partial \boldsymbol{z}_t^{\epsilon}}\Big)+\diag(\boldsymbol{z}_t^{\epsilon})		\frac{\partial^2 L_{\epsilon}^{\text{}}}{\partial \boldsymbol{z}_t^{\epsilon}\partial \boldsymbol{z}_t^{\epsilon} }
	\right)$ is invertible.

	Now we prove  Lemma \ref{lemma.cn.hessian}. By applying  envelop theorem  \citep{afriat1971theory} to   $	l_{\epsilon}(\boldsymbol{\theta}, \boldsymbol{x}_{t}^*, \boldsymbol{u}_{t}^*, \boldsymbol{x}_{t+1}^*)$ in (\ref{equ.contactnet.2}), one can write $	\nabla_{ \boldsymbol{\theta}}L_\epsilon^{\text{}}=\left(
	\frac{\partial
		L_\epsilon^{\text{}} }{\partial \boldsymbol{\theta}}\right)\tran.$
	When taking the second-order derivative, one has
	\begin{equation}\label{app.proof6.equ5}
	\small
	\nabla_{ \boldsymbol{\theta}}^2L_\epsilon^{\text{}}=\sum_{t=1}^{N}\left(
	\frac{\partial^2
		L_\epsilon^{\text{}} }{\partial \boldsymbol{\theta}\partial \boldsymbol{\theta}}+\frac{\partial^2
		L_\epsilon^{\text{}} }{\partial \boldsymbol{\theta}\partial \boldsymbol{z}_t^{\epsilon}}\frac{d \boldsymbol{z}_t^{\epsilon}}{d \boldsymbol{\theta}}
	\right).
	\end{equation}
	Here, by differentiating through the LCP    in (\ref{KKT}), one can obtain 
	\begin{equation}
	\small
	\frac{\partial \boldsymbol{z}_t^{\epsilon}}{\partial \boldsymbol{\theta}}=-	\left(
	\diag\Big(\frac{\partial L_{\epsilon}^{\text{}}}{\partial \boldsymbol{z}_t^{\epsilon}}\Big)+\diag(\boldsymbol{z}_t^{\epsilon})		\frac{\partial^2 L_{\epsilon}^{\text{}}}{\partial \boldsymbol{z}_t^{\epsilon}\partial \boldsymbol{z}_t^{\epsilon} }
	\right)^{-1}\diag(\boldsymbol{z}_t^{\epsilon})	
	\frac{\partial^2 L_{\epsilon}^{\emph{}}}{\partial \boldsymbol{z}_t^{\epsilon}\partial \boldsymbol{\theta} }.
	\end{equation}
	Plugging the above to (\ref{app.proof6.equ5}) leads to (\ref{equ.2nddiff}).
\end{proof}

Lemma \ref{lemma.cn.hessian} states that first,  Hessian   of the violation-based loss  with respect to the system parameter $\boldsymbol{\theta}$ can also be analytically obtained. Such Hessian  is important both for algorithmic implementation and theoretical analysis. Arithmetically, the above Hessian can be used  to develop  second-order methods for optimizing (\ref{equ.loss}). Analytically, the Hessian  can be used to analyze the convexity of the problem. Specifically, if (\ref{equ.contactnet.2}) is convex jointly with respect to $(\boldsymbol{\lambda}_t, \boldsymbol{\phi}_t)$ and $\boldsymbol{\theta}$,  one can show that $L_{\epsilon}(\boldsymbol{\theta}, \mathcal{D})$ will be convex (also see Section 3.2.5 in \citep{boyd2004convex}). This holds for all other system matrices/vectors except matrices $C$ and $F$, which imposes the challenges for learning process. 

Finally, we give the following result showing the violation-based loss $L_{\epsilon}(\boldsymbol{\theta}, \mathcal{D})$ in (\ref{equ.contactnet}) can be controlled to approximate the prediction-based loss $L^{\text{pred}}(\boldsymbol{\theta},\mathcal{D})$ (\ref{equ.prednet})  in terms of both the  loss itself and its differentiability.
\begin{lemma} \label{lemma.approximation}
	Given  $F$ satisfying Assumption \ref{assumption1} and any constant $0<\gamma< \sigma_{\min}(F\tran+F)$, there exists   $\Delta>0$  such that for any $\epsilon\in(0,\Delta]$,
	\begin{itemize}
		\item[(a)]  $L_{\epsilon}(\boldsymbol{\theta},\mathcal{D})$ is differentiable (Lemma \ref{lemma.contactnet}) if    $L^{\emph{pred}}(\boldsymbol{\theta},\mathcal{D})$ is differentiable (Lemma \ref{lemma.lcpdiff}).
		\item[(b)]
$
		L_{\epsilon}(\boldsymbol{\theta},\mathcal{D})\rightarrow L^{\emph{pred}}(\boldsymbol{\theta},\mathcal{D})\quad \text{as}\quad \epsilon \rightarrow 0.
$
		
	\end{itemize}
\end{lemma}

\begin{proof}
	We here only provide the  sketch for the proof of the above lemma due to page limits. 
	In the proof of claim (a), first, we can show that the strict complementarity in Lemma (\ref{lemma.lcpdiff}) is equivalently to say the  strict complementarity for (\ref{equ.cn.keyequ}) in Lemma \ref{lemma.cn.key}, i.e.,  the following LCP
	\begin{equation}\label{KKT2}
			\boldsymbol{0}\leq \small\left(\frac{\partial f}{\partial \boldsymbol{z}_t}\right)^{\prime}\perp\boldsymbol{z}_t^{}\geq \boldsymbol{0}\quad\text{with}\quad \vec{z}_t:=(\vec{\lambda}_t, \vec{\phi}_t) \quad\text{and}\quad  f(\boldsymbol{\lambda},\boldsymbol{\phi})\defeq\boldsymbol{\lambda}\tran\boldsymbol{\phi} +
			\frac{1}{\gamma}
			\norm{F\boldsymbol{\lambda}+\boldsymbol{q}_t-\boldsymbol{\phi}}^2
\end{equation}
is strict complementarity. Further, one can show that (\ref{KKT}) will converge to (\ref{KKT2}) as $\epsilon\rightarrow 0$. Since the strict complementarity preserves as $\epsilon$ falls in a small neighborhood around $0$ (this is similar to the proof of Theorem 2.1 in \citep{fiacco1976sensitivity}), one can say   $L_{\epsilon}(\boldsymbol{\theta},\mathcal{D})$ is differentiable with any small $\epsilon>0$ in the neighborhood around 0. The proof of claim (b) can directly follow the standard proof in penalty-based optimization \citep{fiacco1990nonlinear}.
\end{proof}

The above lemma has shown that the proposed violation-based $L_{\epsilon}(\boldsymbol{\theta}, \mathcal{D})$ and prediction-based loss  $L^{\text{pred}}(\boldsymbol{\theta},\mathcal{D})$ are essentially related to each other  in term of function itself and its differentiability with respect to $\boldsymbol{\theta}$. Specifically, the differentiability of prediction-based loss $L^{\text{pred}}(\boldsymbol{\theta},\mathcal{D})$ in Lemma \ref{lemma.lcpdiff}, i.e., strict complementarity for the LCP, always implies the differentiability of the violation-based  loss $L_{\epsilon}(\boldsymbol{\theta},\mathcal{D})$ for any choice of small $\epsilon>0$. Second, by controlling   $\epsilon \rightarrow 0$, the violation-based formulation  approximates to  prediction-based one. 

\medskip

In light of  all properties stated above, we now summarize the advantage of the proposed violation-based loss $L_{\epsilon}(\boldsymbol{\theta},\mathcal{D})$ over the prediction-based loss $L^{\text{pred}}(\boldsymbol{\theta},\mathcal{D})$.  First, $L_{\epsilon}(\boldsymbol{\theta},\mathcal{D})$ allows for more efficient computation of gradient, while  differentiation of $L_{\epsilon}(\boldsymbol{\theta},\mathcal{D})$ always requires the matrix inverse.   Second, $L_{\epsilon}(\boldsymbol{\theta},\mathcal{D})$ permits analytical Hessian information, which is important both for the algorithmic implementation and theoretical analysis, such Hessian matrix is difficult to obtain for  $L^{\text{pred}}(\boldsymbol{\theta},\mathcal{D})$. Finally, the violation-based loss $L_{\epsilon}(\boldsymbol{\theta},\mathcal{D})$ is flexible to approximate   $L^{\text{pred}}(\boldsymbol{\theta},\mathcal{D})$ in terms of both the  function itself and its differentiability with respect to $\boldsymbol{\theta}$,  by controlling the weight parameter $\epsilon$.

\section{Examples}

 In implementation, one way to enforce  Assumption \ref{assumption1} is using re-parameterization tricks:  by re-parameterizing $F=GG\tran+H-H\tran$ with any matrices  $G\in\mathbb{R}^{n_{\lambda}\times n_{\lambda}}$ and $H\in\mathbb{R}^{n_{\lambda}\times n_{\lambda}}$, one can  easily see that $F+F\tran\succeq0$. 
Also note that  $F$ and  $C$ in a LCS  (\ref{equ.lcs}) are permutation- and scaling- invariant with respect to  $\mathcal{D}=\{(\boldsymbol{x}_t^*,\boldsymbol{u}_t^*, \boldsymbol{x}_{t+1}^*)\}_{t=1}^N$.  Specifically, if  $(\vec{x}_t^*, \vec{u}_t^*, \vec{x}_{t+1}^*)$ satisfies  (\ref{equ.lcs}) with unobserved $\vec{\lambda}_t^*$, it also satisfies the following LCS with  $\boldsymbol{\tilde{\lambda}}_t^*=PS\vec{\lambda}_t$. 
\begin{equation}
 \begin{aligned}
 &\boldsymbol{x}_{t+1}^*=A\boldsymbol{x}_t^*+B\boldsymbol{u}_t^*+CS^{-1}P\tran\boldsymbol{\tilde{\lambda}}_t^*+\boldsymbol{d},\\
 &\boldsymbol{0}\leq\boldsymbol{\tilde{\lambda}}_t^* \perp PSD\boldsymbol{x}_t  +PSE\boldsymbol{u}_t + PSFS^{-1}P\tran\boldsymbol{\tilde{\lambda}}_t^* + \boldsymbol{c} \geq \boldsymbol{0}.
 \end{aligned}
\end{equation}
 Here $P\in\mathbb{R}^{n_\lambda \times n_\lambda}$,   is any permutation matrix and $S$ is any diagonal matrix with positive diagonal entries. To mitigate this ambiguity, we  add a regularizing cost   $R(\boldsymbol{\theta})=\omega\norm{C}_F^2$, where  $\norm{\cdot}_F$ is the matrix Frobenius norm and  $\omega$ is the weighting parameter.  We set $\omega=10^{-5}$ in our following experiments for both methods.

We randomly define a ground-truth LCS with $\boldsymbol{\theta}^*$, where all parameters are  selected from a uniform distribution in range $[-1, 1]$. To generate  training data $\mathcal{D}=\{(\boldsymbol{x}_t^*,\boldsymbol{u}_t^*, \boldsymbol{x}_{t+1}^*)\}_{t=1}^{N_{\text{train}}}$, we sample $\boldsymbol{x}^*_t$ and $\boldsymbol{u}^*_t$  from  uniform distributions over  $[-10,10]$ and  $[-5,5]$, respectively, and then solve  $\boldsymbol{x}^*_{t+1}$ based on $\boldsymbol{\theta}^*$. We also add zero-mean Gaussian noise with standard deviation $\sigma{=}{10^{-2}}$ to  $\mathcal{D}$. 
We  generate similar, but noiseless,  testing data $\mathcal{T}=\{(\boldsymbol{\bar{x}}_t,\boldsymbol{\bar{u}}_t, \boldsymbol{\bar{x}}_{t+1})\}_{t=1}^{N_{\text{test}}}$.
To evaluate  the learned LCS on  $\mathcal{T}$,  we define the following \emph{mean relative prediction error}, 
\begin{equation}
\small
	e_{\text{test}}=\frac{\sum_{t=1}^{N_{\text{test}}}\norm{{\boldsymbol{x}_{t+1}^{\vec{\theta}}-\boldsymbol{\bar{x}}_{t+1}}}^2}{\sum_{t=1}^{N_{\text{test}}}\norm{\boldsymbol{\bar{x}}_{t+1}}^2},
\end{equation}
with $\boldsymbol{x}_t^{\boldsymbol{\theta}}$  the   state predicted  by the learned LCS  at  $(\boldsymbol{\bar{x}}_t,\boldsymbol{\bar{u}}_t).
$ The size of $\mathcal{T}$ is $N_{\text{test}}{=}1000$.

The following simulations will evaluate  different aspects of the proposed violation-based learning formulation (\ref{equ.contactnet}), in comparison with the prediction-based learning formulation (\ref{equ.prednet}). For any learning formulation, 
each evaluation includes a total of 30 training rounds (otherwise stated), and each training round uses a random ground-truth LCS of $\boldsymbol{\theta}^*$ to generate  $\mathcal{D}$ and $\mathcal{T}$ and uses a  random initialization $\boldsymbol{\theta}$ to initialize the training. Each training round uses the Adam algorithm \citep{kingma2014adam} with the mini-batch size 200 and the learning rate $10^{-3}$ (other Adam parameters: $\beta_1=0.9$, $\beta_2=0.9$, $\epsilon=10^{-6}$). 
Because we randomize the LCS systems, and some systems may be easier to identify than others, we expect fairly high variance of the results.

\subsection{Results and Analysis}

The evaluations of  different aspects of the proposed violation-based formulation (\ref{equ.contactnet}) versus the prediction-based   formulation (\ref{equ.prednet}) are shown in Fig. \ref{fig}. In summary of all evaluations, one can conclude that the proposed violation-based learning outperforms the prediction-based learning, specifically  when (a) handling high numbers of system modes, e.g.,  $16$k modes at $n_{\lambda}=20$  as shown in Fig. \ref{fig:a}; (b) dealing with high system dimensions, e.g., $n_x=128$ as shown in Fig. \ref{fig:b}; and (c) learning \emph{ high-stiff LCS system} as   shown in Fig. \ref{fig:c}. Further, Fig. \ref{fig:d} and Fig. \ref{fig:d} show that  parameters $\gamma$ and $\epsilon$ in the violation-based loss (\ref{equ.contactnet}) are not sensitive to the performance, and thus finding  proper $\gamma$ and $\epsilon$ is not difficult in practice.

\begin{figure}[h]
	\centering     
	\subfigure[\small Varying  $n_{\lambda}$, with $n_x{=}10$ and $n_u{=}4$. ]{\label{fig:a}\includegraphics[width=0.32\textwidth]{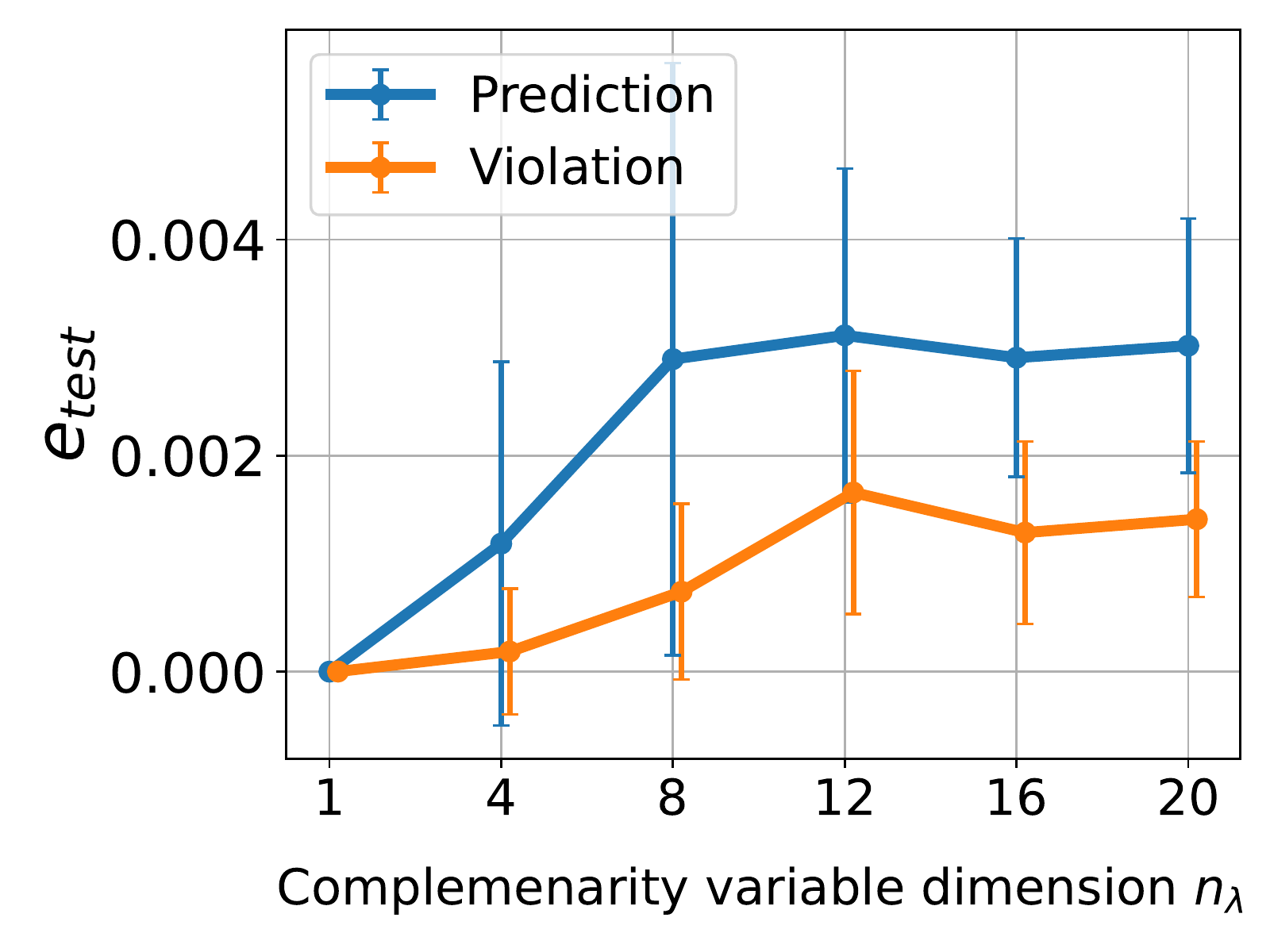}}
	\hfill
		\subfigure[\small Varying   $n_x$, with $n_u{=}4$ and $n_{\lambda}{=}10$.
	]{\label{fig:b}\includegraphics[width=0.325\textwidth]{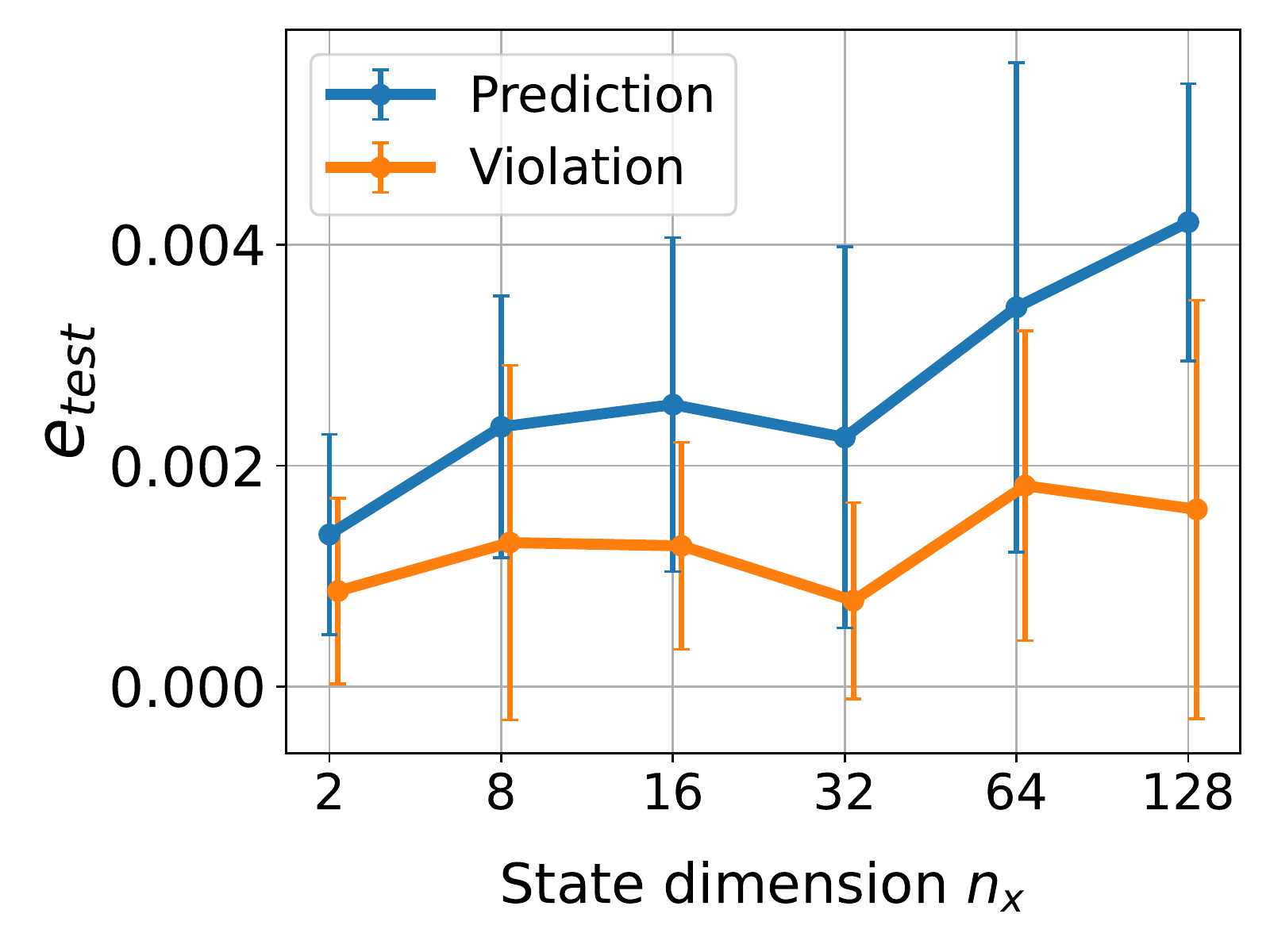}}
	\hfill
	\subfigure[\small Varying  $\small\sigma_{\min}\small(F{+}F\tran)$, with $n_x{=}8$, $n_u{=}2$ and $n_{\lambda}{=}10$.
	]{\label{fig:c}\includegraphics[width=0.325\textwidth]{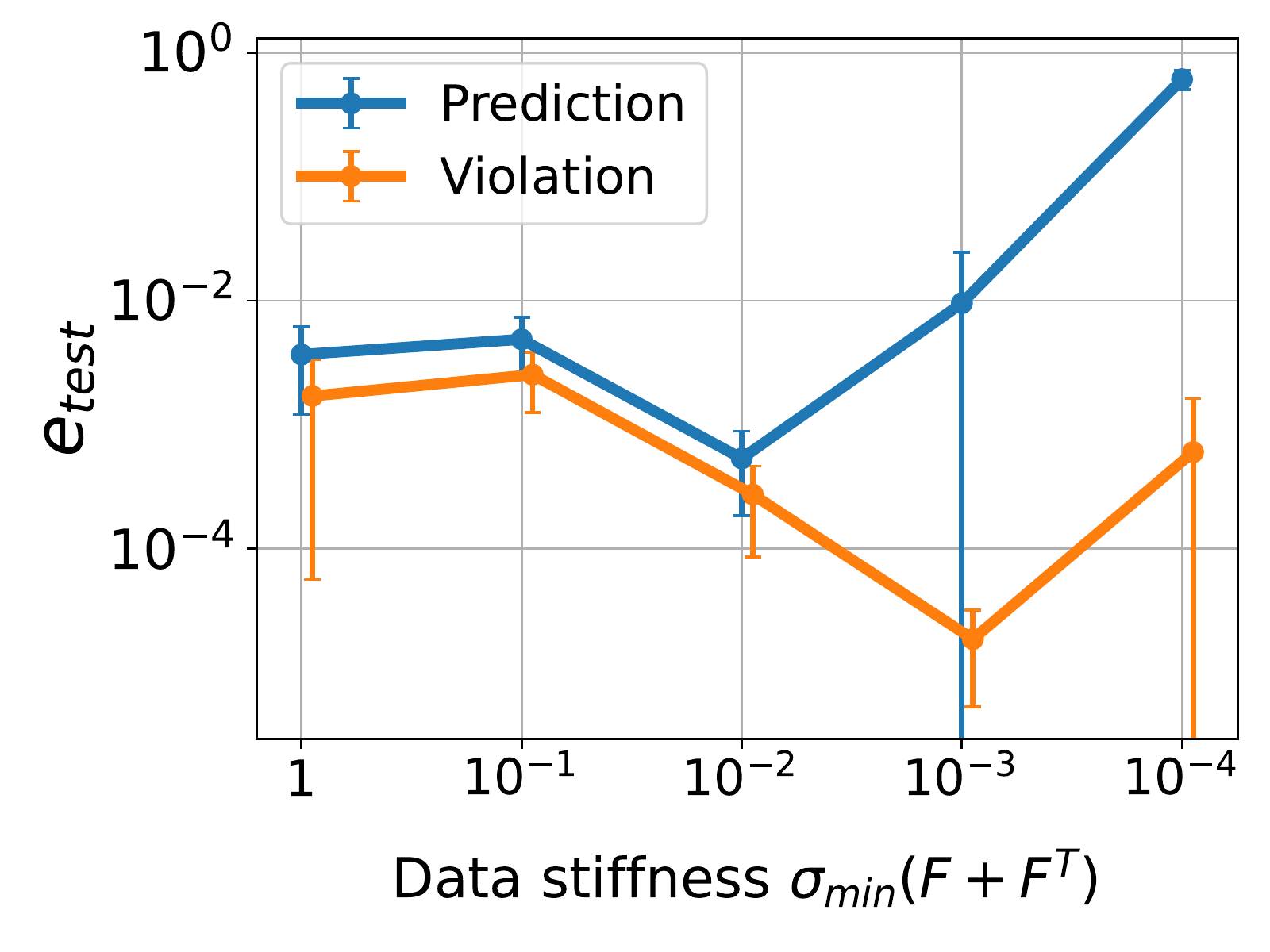}}
	\subfigure[\small  Varying $\gamma$.  $\small\sigma_{\min}\small(F{+}F\tran){=}1$, $n_x=4$, $n_{u}{=}2$, and $n_{\lambda}{=}4$. ]{\label{fig:d}\includegraphics[width=0.325\textwidth]{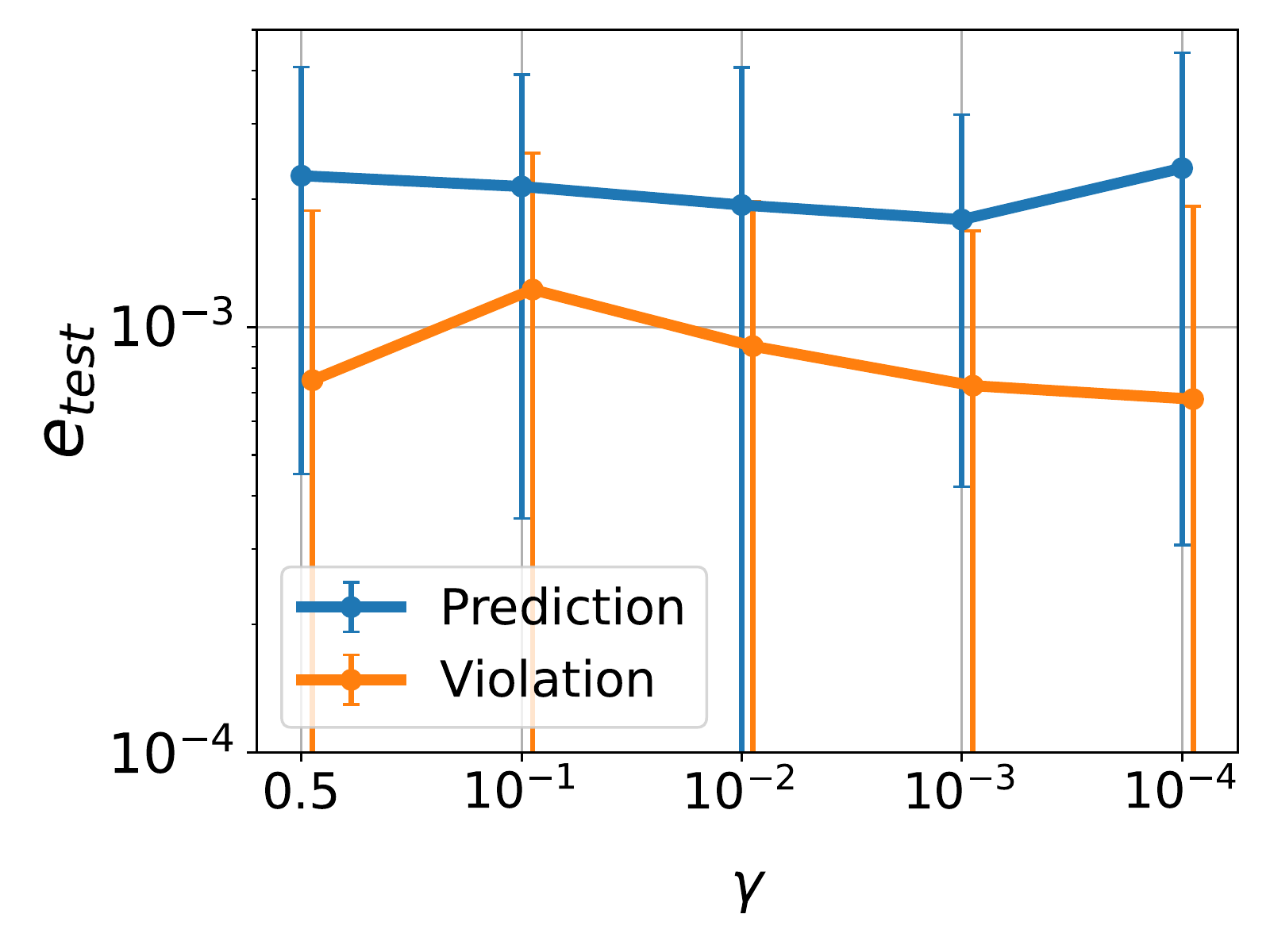}}
    \quad
	\subfigure[\small  Varying $\epsilon$.  $\small\sigma_{\min}\small(F{+}F\tran){=}1$, $n_x=4$, $n_{u}{=}2$, and $n_{\lambda}{=}4$. ]{\label{fig:e}\includegraphics[width=0.325\textwidth]{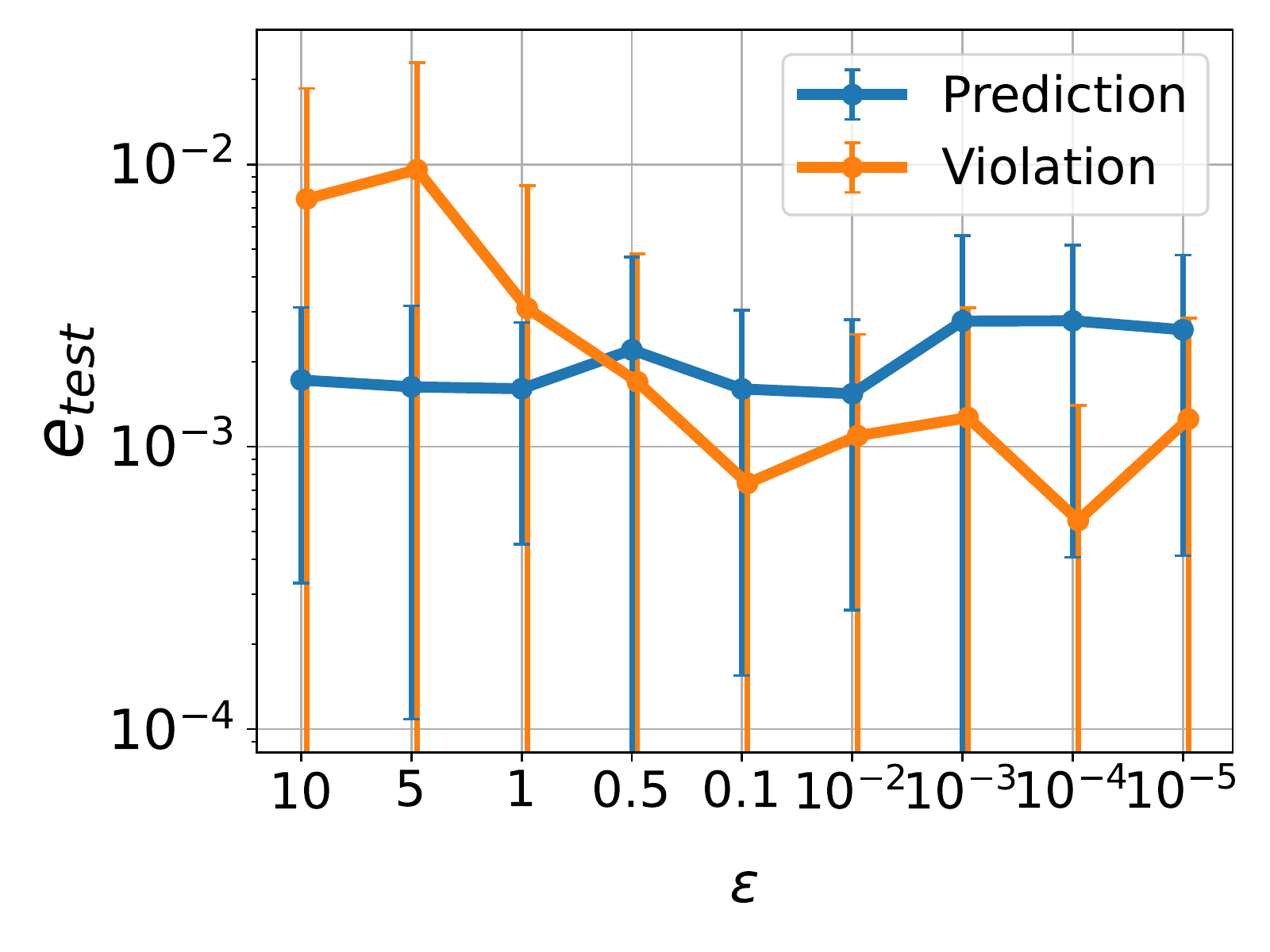}}
	
	\caption{Evaluations of different aspects of  the proposed violation-based learning (\ref{equ.contactnet})  in comparison with prediction-based learning (\ref{equ.prednet}).  }
	\label{fig}
\end{figure}

(a)  In Fig. \ref{fig:a}, we vary the number of complementarity constraints, i.e.,  $n_\lambda$, with fixed    $n_{x}=10$ and $n_{u}=4$.  $\mathcal{D}$ has the size of $N_{\text{train}}=50$k for all evaluations, and $\small\sigma_{\min}\small(F{+}F\tran)=1$. Note that the maximum achievable number of modes  depends on both  $n_{\lambda}$  and $n_{x}$, specifically is $2^{n_{\lambda}}$ if   $n_{x}=n_{\lambda}$. In our above evaluations, 
at $n_{\lambda}=20$,   $\mathcal{D}$ contains around 16k modes. In the violation-based loss, $\gamma=10^{-2}$ and $\epsilon=10^{-4}$. Fig. \ref{fig:a} shows an obvious advantage of the  violation-based method, especially for learning LCS with   high number of  modes, e.g., $16$k modes. This ability of the violation-based method  is also in contrast to the prior PWA work which only handles tens of modes, as discussed in introduction.

(b)  In Fig. \ref{fig:b}, we vary system state dimension $n_x$ with $n_{\lambda}=10$ and $n_{u}=4$. Here, each evaluation includes  15 training rounds, and other training settings follow  (a). The results show the advantage of the proposed violation-based method over the prediction-based method in handling high dimension of system states, such as $n_x\geq64$.

(c) Fig. \ref{fig:c} varies the system stiffness indicated by $\small\sigma_{\min}\small(F{+}F\tran)$, i.e., smaller $\small\sigma_{\min}\small(F{+}F\tran)$ means a stiffer system. Here, $n_x{=}8$, $n_u{=}2$, $n_{\lambda}{=}10$, and others follow (a). The results show \emph{a signifcant advantage of  the violation-based formulation (\ref{equ.contactnet}) over prediction-based learning (\ref{equ.prednet})}. By weighting  the prediction error against the LCP violation,  loss  (\ref{equ.contactnet})   generalizes better in the presence of stiffness and can be 
connected to  graph distance \citep{bianchini2021generalization}.

(d-e) Fig. \ref{fig:d} and \ref{fig:e} plot the performance
of the  violation-based  method for  different choices of 
 $\gamma$ and $\epsilon$, respectively. Here $n_x{=}4$, $n_{\lambda}{=}4$,  $n_u{=}2$,  $N_{\text{train}}{=}5$k, and others follow (a). Fig. \ref{fig:b} confirms Lemma \ref{lemma.cn.key} that any $0{<}\gamma{<}\sigma_{\min}\small(F{+}F\tran)$ does not significantly influence the  results. Fig. \ref{fig:e}, investigating the choice of $\epsilon$, shows, across many orders of magnitude, the results are largely invariant to this choice. For small $\epsilon$, following Lemma \ref{lemma.approximation}, we see the violation-based formulation perform similarly to the prediction-based form. Large $\epsilon$ can lead to results which grossly violate complemenatrity, and thus perform poorly.  Empirically, it is not difficult to find a $\epsilon$ in the middle range which does not much influence    the  performance.

\vspace{-5pt}
\section{Conclusion}
\vspace{-5pt}
In this paper, we have proposed a violation-based loss formulation which enables to learn a LCS using  gradient-based methods. The  violation-based loss is a sum of  dynamics prediction loss and a   novel complementarity violation loss.  We  have  shown several some properties  attained by this loss formulation.  The numerical results demonstrate a state-of-the-art ability to identify piecewise-affine dynamics, outperforming methods which must differentiate through  non-smooth linear complementarity problems.

\vspace{-5pt}

\acks{Toyota Research Institute provided funds to support this work. This work was also supported by the National Science Foundation under Grant No. CMMI-1830218 and an NSF Graduate Research Fellowship under Grant No. DGE-1845298.}

\vspace{-5pt}

\appendix
\subsection*{Appendix: Proof of Lemma \ref{lemma.lcpdiff}} \label{appendix.proof.lemma1}

\vspace{-5pt}
We first prove  if  $\boldsymbol{\lambda}_t^*=\lcp(F,D\boldsymbol{x}_t^*  +E\boldsymbol{u}_t^*+\vec{c})$ is the strictly complementarity, then matrix $S_t\defeq\diag(D\boldsymbol{x}_t^*  +E\boldsymbol{u}_t^* + F\boldsymbol{\lambda}_t^{*} + \boldsymbol{c}) +\diag(\boldsymbol{\lambda}_t^*)F$ 
is invertible. We prove this by contradiction. 
Suppose  $S_t$ is singular, and there exists a non-zero $\boldsymbol{v}\in\mathbb{R}^{n_\lambda}$ s.t. $ S_t\tran \boldsymbol{v}=\boldsymbol{0}$. 
Since $F$ satisfying Assumption \ref{assumption1} is the P-matrix, so is $F\tran$. Consider the two   cases. If   $\diag(\boldsymbol{\lambda}_t^{*}) \boldsymbol{v}=\boldsymbol{0}$,  thus $\diag(D\boldsymbol{x}_t^*  +E\boldsymbol{u}_t^* + F\boldsymbol{\lambda}_t^{*} +  \boldsymbol{c})\boldsymbol{v}=\boldsymbol{0}$. There must exist $i\in\{1,...,n_{\lambda}\}$ such that $\boldsymbol{\lambda}_t^{*}[i]=0$ and $\boldsymbol{v}[i]\neq0$. By the strict complementarity, $	(D\boldsymbol{x}_t^*  +E\boldsymbol{u}_t^* + F\boldsymbol{\lambda}_t^{*} + \boldsymbol{c})[i] \cdot \boldsymbol{v}[i]\neq0,$
which contradicts  $\diag(D\boldsymbol{x}_t^*  +E\boldsymbol{u}_t^* + F\boldsymbol{\lambda}_t^{*} +  \boldsymbol{c})\boldsymbol{v}=\boldsymbol{0}$. If  $\diag(\boldsymbol{\lambda}_t^{*}) \boldsymbol{v}\neq\boldsymbol{0}$,  we  have $	F\tran \diag(\boldsymbol{\lambda}_t^{*}) \boldsymbol{v}=-\diag(D\boldsymbol{x}_t^*  +E\boldsymbol{u}_t^* + F\boldsymbol{\lambda}_t^{*} + \boldsymbol{c}) \boldsymbol{v}.$
Then, for all $i\in\{1,2,...,n_\lambda\}$, $			\left(\diag(\boldsymbol{\lambda}_t^{*})\boldsymbol{v}\right)[i]\cdot \left(F\tran \diag(\boldsymbol{\lambda}_t^{*}) \boldsymbol{v}\right)[i]
=\left(\diag(\boldsymbol{\lambda}_t^{*})\boldsymbol{v}\right)[i]\cdot \left(-\diag(D\boldsymbol{x}_t^*  +E\boldsymbol{u}_t^* + F\boldsymbol{\lambda}_t^* + \boldsymbol{c}) \boldsymbol{v}\right)[i]=0.$ In fact, since $F\tran$ is a P-matrix, the above result contradicts with the reverse-sign property of P-matrix (see Theorem 3.3.4 in \citep{cottle2009linear}). Combine the above two cases, we  conclude that $S_t$ is non-singular.

Next, we prove  Lemma \ref{lemma.lcpdiff}. 
Define $\boldsymbol{g}(\boldsymbol{\lambda}_t, D, E, F, \boldsymbol{c})=	\diag(\boldsymbol{\lambda}_{t})\left(D\boldsymbol{x}_t^*  +E\boldsymbol{u}_t^* + F\boldsymbol{\lambda}_t + \boldsymbol{c}\right)=\boldsymbol{0}.$
It is obvious  that  $\boldsymbol{\lambda}_t^*=\lcp(F,D\boldsymbol{x}_t^*  +E\boldsymbol{u}_t^*+\vec{c})$  satisfies  the above equation. Next, we take the Jacobian Matrix of $\boldsymbol{g}(\boldsymbol{\lambda}_t, D, E, F, \boldsymbol{c})$ with respect to $\boldsymbol{\lambda}_t$ evaluated at $\boldsymbol{\lambda}^*$, leading to $	\frac{\partial \boldsymbol{g}}{\partial \boldsymbol{\lambda}_t}|_{\boldsymbol{\lambda}_t^*}=	\diag(D\boldsymbol{x}_t^*  +E\boldsymbol{u}_t^* + F{\boldsymbol{\lambda}_t^*} + \boldsymbol{c}) +\diag({\boldsymbol{\lambda}_t^*})F=S_t.$
Since $S_t$
is invertible due to the previous proof, by applying the implicit function theorem \citep{rudin1976principles}, one can reach the differentiability in Lemma \ref{lemma.lcpdiff}. This completes the proof. \qed

\bibliography{mybib}

\end{document}